\documentclass[acmsmall,table]{acmart}

\usepackage{amsmath,amsthm,amsfonts}
\usepackage{algorithmic}
\usepackage{textcomp}
\usepackage{bm}
\usepackage{graphicx}
\usepackage{subcaption}
\usepackage{siunitx}

\usepackage{mathtools}

\DeclareMathOperator*{\diag}{diag}

\usepackage{lipsum}

\newtheorem{theorem}{Theorem}[section]
\newtheorem{prop}[theorem]{Proposition}
\newtheorem{definition}[theorem]{Definition}
\newtheorem{remark}[theorem]{Remark}
\newtheorem{assumption}{Assumption}

\AtBeginDocument{%
  }

\setcopyright{acmlicensed}
\copyrightyear{2025}
\acmYear{2025}
\acmDOI{10.1145/3725857}

\acmJournal{JACM}
\acmVolume{16}
\acmNumber{3}
\acmArticle{68}
\acmMonth{5}

\begin{document}

\title{A Gated Graph Neural Network Approach to Fast-Convergent Dynamic Average Estimation}

\author{Antonio Marino}
\email{antonio.marino@irisa.fr}
\orcid{1234-5678-9012}
\affiliation{%
  \institution{Universitè de Rennes, CNRS, Inria, IRISA}
  \city{Rennes}
  \country{France}
}

\author{Claudio Pacchierotti}
\affiliation{%
	\institution{CNRS}
	\city{Rennes}
	\country{France}
}
\email{claudio.pacchierotti@cnrs.fr}

\author{Paolo Robuffo Giordano}
\affiliation{%
	\institution{CNRS}
	\city{Rennes}
	\country{France}
}
\email{prg@cnrs.fr}

\renewcommand{\shortauthors}{Marino et al.}

\begin{abstract}
 Dynamic average estimation is a critical problem in multi-agent systems, enabling agents to collaboratively estimate time-varying signals using only local information exchange. Traditional model-based approaches often face challenges related to convergence speed and sensitivity to network topology changes. This paper introduces a novel learning-based solution leveraging Gated Graph Neural Networks (GGNNs) for fast-convergent dynamic average estimation in a fully distributed manner. Taking advantage of the inherent structure of GGNNs, the proposed method models the estimation process as a distributed autoregressor, ensuring rapid convergence while maintaining stability. We incorporate a regularization term during training to enforce convergence guarantees and introduce an encoding-decoding mechanism to reduce communication overhead without sacrificing accuracy compared to standard GGNNs. Extensive numerical experiments demonstrate that our approach significantly outperforms conventional model-based estimators in terms of both convergence speed and precision, making it a promising alternative for multi-agent applications that require dynamic average estimation.
\end{abstract}



\received{}
\received[revised]{}
\received[accepted]{}

\maketitle

\section{Introduction}
Within a communication network consisting of multiple agents, a prevalent challenge faced by each agent involves the localized tracking of dynamic average values derived from time-varying signals dispersed across the group. When the signals are not available to the single agent, addressing this problem necessitates the use of distributed dynamic average estimation algorithms~\cite{kia2019tutorial}. These algorithms rely solely on local information exchange and are applicable regardless of the size of the group. Dynamic average estimation finds applications in various multi-agent scenarios, including formation control~\cite{porfiri2007tracking, listmann2009consensus}, distributed estimation~\cite{olfati2005distributed, wang2017diffusion}, connectivity control~\cite{robuffo2013passivity} and distributed optimization~\cite{wang2011control, zhu2011distributed}. To illustrate a practical application, consider a multi-robot system where a team of autonomous drones needs to estimate and track the average wind speed across a distributed environment. Each drone measures local wind conditions and shares information with its neighbours using a dynamic average estimation. This allows the system to converge to an accurate global estimate, preferably despite network topology changes, enabling the drones to optimize their flight paths in real-time. Similar applications can be found in smart grids, where multiple sensors estimate real-time power consumption to improve load balancing and efficiency without centralized computation.

State-of-the-art approaches rely on consensus algorithms with special initialization requirements and signal derivative knowledge or integral action for robustness. Other methods focus on handling dynamic network topologies but face issues like chattering and slow convergence. In general, existing methods fail to reach fast convergence while remaining robust to graph topology variations and adapted for static and dynamic signals. Recently, graph neural networks (GNNs) have shown good performances to distributed consensus, though they require significant communication overhead. To address these challenges, we propose a novel learning-based approach using Gated Graph Neural Networks (GGNNs) to achieve fast-convergent dynamic average estimation in a fully distributed manner. By leveraging the native structure of GGNNs, our method ensures that agents can predict the dynamic average using a trained neural model while maintaining stability and convergence guarantees.

The key contributions of this work are as follows:
\begin{itemize}
    \item We introduce a GGNN-based learning model for dynamic average estimation that achieves rapid convergence while preserving stability.
    \item We formally analyze the stability properties of the proposed approach and introduce a regularization term to enforce convergence conditions during training.
    \item We incorporate an encoding-decoding mechanism to reduce communication overhead while maintaining estimation accuracy.
    \item We validate our approach through extensive numerical experiments, demonstrating superior performance compared to traditional model-based estimators in terms of convergence speed and precision.
\end{itemize}

\section{Related Work}
The average consensus estimation (ACE) problem formalization and one of the earliest solutions were presented for the first time in~\cite{spanos2005dynamic}. Their algorithm initially estimates a static average by applying consensus on weighted-balanced graphs and then incorporates the local signal derivative in the estimation update. However, this algorithm suffers from some limitations, such as requiring special initialization for every link creation or drop and the knowledge of the signal derivative. Freeman et al.\cite{freeman2006stability} proposed an improved version of Spanos et al.'s algorithm, introducing an integral action that eliminates the need for special initialization of the estimation and enhances its robustness against bounded noise in agent signals. Furthermore, they relaxed the requirement for the signal derivative through an appropriate change of state variables. Although the algorithm in\cite{freeman2006stability} guarantees asymptotic convergence, its tracking performance depends on both the rates of reference signals and the connectivity of the communication graph. If the signal varies too quickly and/or the graph is too loosely connected, the estimated average tracking performance can easily deteriorate.

Several other approaches have been proposed to improve the performance of dynamic average estimation. For example, Chen et al.\cite{chen2012distributed} introduced a scheme that enhances tracking accuracy by utilizing local estimations and consensus protocols tailored to improve performance in dynamically changing environments. Similarly, Nosrati et al.\cite{nosrati2012dynamic} proposed a method focusing on tracking accuracy but requiring specific initialization steps to ensure convergence. These schemes, while effective in certain conditions, encounter challenges when network topology changes, such as when agents join or leave the network, necessitating reinitialization to prevent non-zero steady-state errors.

To mitigate these issues, robust solutions have been explored. Chen et al.\cite{chen2013robust} and Xu et al.\cite{xu2021robust} investigated approaches employing increasing gains and sign functions designed to handle variations in network topology more gracefully. These robust solutions aim to enhance the stability and reliability of dynamic average estimations even in the presence of network disturbances. However, a common drawback is the phenomenon of chattering, where the algorithm oscillates around the true average value, or slow convergence, which can significantly delay the estimation process. These limitations reduce the effectiveness of the algorithms in real-time applications where quick adaptation is crucial.

In response to these challenges, a novel method by Stamouli et al.~\cite{stamouli2019robust} utilizes prescribed functions to achieve convergence with specified performance criteria. This method provides a structured approach to ensuring that the convergence of the average estimation meets predefined standards. However, the practicality of this approach is somewhat limited, as it requires re-tuning when the reference signals change. This re-tuning process can be cumbersome and time-consuming, particularly in dynamic environments where signal changes are frequent.

Moreover, the prescribed functions approach demands a comprehensive understanding of the system dynamics and careful adjustment of the parameters to maintain optimal performance. Despite its potential for high accuracy, the need for constant re-tuning poses a significant barrier to its widespread application, especially in scenarios where the operating conditions are highly variable or unpredictable.

Sandryhaila et al.\cite{sandryhaila2014finite} were the first to demonstrate an interest in achieving finite-time consensus using graph filters. They introduced a finite impulse response (FIR) graph filter that was adjusted based on the spectral characteristics of the graph. However, this approach becomes impractical in distributed systems due to the requirement of knowing the structure of the graph Laplacian matrix and its spectral decomposition. In contrast, Iancu et al.\cite{iancu2021towards} suggested employing a \textit{graph neural network} (GNN) that addresses the limitations of linear graph filter approaches by offering enhanced adaptability to graphs of varying sizes and connectivity levels.

GNNs are used for prediction and analysis tasks on graphs. They provide a convenient topological representation for a wide range of multi-agent problems, including decision-making~\cite{I-Lingfei2022GNNFoundations}, flocking control~\cite{gama2021graph}, space coverage~\cite{li2021message}, and multi-robot path planning~\cite{tolstaya2021multi}. Even if promising, the neural network proposed in~\cite{iancu2021towards} uses many GNN layers, requiring communication of a considerable amount of variables, which depends on the number of GNN layers and the chosen neural network features.

This work proposes a distributed neural network based on Gated Graph Neural Networks (GGNNs)~\cite{ruiz2020gated} that allows solving the distributed average estimation problem regardless of the team size within a few iterations. We achieve this result by imposing the incremental convergence of the neural state and consequently of the average estimation. In this way, the problem of distributed average estimation can be formalized more naturally as a dynamic regressor without the use of many neural layers.

%%%%%%%%%%%%%%%%%%%%%%%%%%%%%%%%%%%%%%%%%%%%%%%%%%%%%%%%%%%%%%%%%%%%%%%%%%%%%%%%
\section{Preliminaries}
\label{Preliminaries}
Let $ \mathcal{G}=(\mathcal{V},\,\mathcal{E})$ be an undirected graph where $\mathcal{V} = \{v_1, \dots, v_N\}$ is the vertex set (representing the $ N$ agents in the group) and $ \mathcal{E}\subseteq \mathcal{V}\times \mathcal{V}$ is the edge set.  
Each edge ${ e_k=(i,\,j)\in\mathcal{E}}$ is associated with a weight $ w_{ij}=w_{ji} \geq 0$ such that $ w_{ij}>0$ if the agents $ i$ and $ j$ can interact and $ w_{ij}=0$ otherwise. 
As usual, we denote with $ \mathcal{N}_i=\{j\in\mathcal{V}|\;w_{ij}>0\}$ the set of neighbors of agent $ i$.
We also let $ \bm{A} \in \mathbb{R}^{N\times N}$ be the adjacency matrix with the off diagonal entries given by the weights $ w_{ij}$. 
%The incidence matrix of the graph is defined as $\bm{B} = [b_{ki}]$, where $b_{ki}$ is $\{w_{ij}, -w_{ij}\}$ respectively if the edge $e_k$ enters or leaves node $i$, and $b_{ki} = 0$ otherwise. 

Defining the degree matrix $ \bm{D} =\diag(d_i)$ with $ d_i = \sum_{j\in \mathcal{N}_i} w_{ij}$, the Laplacian matrix of the graph is $ \bm{L}=\bm{D}-\bm{A}$.
The graph signal $\bm{x} \in \mathbb{R}^N$, whose $i$-component $x_i$ is assigned to agent $i$, can be processed over the network by the following linear combination rule applied by each agent
\begin{equation}
	\boldsymbol\ell_i\bm{x} = \sum_{j\in \mathcal{N}_i} w_{ji} (x_i - x_j), 
	\label{eq:aggregation}
\end{equation}
where $ \boldsymbol\ell_i$ is the $ i$-th row of $ \bm{L}$. This process is also known as \textit{aggregation} in the graph signal processing literature. The aggregation can be operated by means of any support matrix $ \bm{S}$, e.g., Laplacian, adjacency matrix, and so forth, which respects the sparsity pattern of the graph.

Performing $ k$ repeated applications of $ \bm{S}$ on the same signal represents the aggregation of the $ k$-hop neighborhood information. In analogy with traditional signal processing, the application of $ \bm{S}$ can be used to define a linear graph filtering~\cite{P-Shuman2013SPG} that processes the multi-feature signal $ \bm{x} \in \mathbb{R}^{N \times G}$ with $ G$ features:
\begin{equation}
 H_{\bm{S}}(\bm{x}) = \sum_{k=0}^K \bm{S}^k \bm{x} \bm{H_k},
	\label{eq:filter}
\end{equation}
where the weights $ \bm{H_k} \in \mathbb{R}^{G \times F}$ define the filter that transforms $\bm{x}$ in a new graph signal of $F$ features. Note that ${ \bm{S}^k=\bm{S}(\bm{S}^{k-1})}$, so that it can be computed locally with repeated 1-hop communications between a node and its neighbors. Hence, the computation of $ H_{\bm{S}}$ is naturally distributed over the network nodes. 
\subsection{Graph Neural Network}
\label{P-GNN}
Although $ H_{\bm{S}}$ is simple to evaluate, it can only represent a linear mapping between input and output graph signals. GNNs increase the expressiveness of the linear graph filters by means of pointwise nonlinearities $ \rho$, e.g., tanh, ReLU etc., following a filter bank. Letting $ H_{\bm{S}l}$ be a bank of $ F_{l-1} \times F_l$ filters at layer $ l$, the GNN layer is defined as
\begin{equation}
	 \bm{x}_l = \rho(H_{\bm{S}l}(\bm{x}_{l-1})), \qquad  \bm{x}_{l-1} \in \mathbb{R}^{N \times F_{l-1}}.
\end{equation}
Starting by $ l=0$ with $ F_0$, the signal tensor $ \bm{x}_{l_n} \in \mathbb{R}^{N \times F_{l_n}}$ is the output of a cascade of $ l_n$ GNN layers. During the training, this model learns the graph filter weights.
\subsection{Gated Graph Neural Network}
\label{P-GGNN}
Recurrent models of GNNs can solve time-dependent problems. These models, similarly to recurrent neural networks (RNNs), are known as graph recurrent neural networks (GRNNs). GRNNs utilize memory to learn patterns in data sequences, where the data is spatially encoded within graphs, regardless of the team size of the agents. However, traditional GRNNs encounter challenges such as vanishing gradients, which are also found in RNNs. Additionally, they face difficulties in handling long sequences in space, where certain nodes or paths within the graph might be assigned more importance than others in long-range exchanges, causing imbalances in the graph's information encoding. \\
%This problem is connected to the Laplacian over smoothing which accumulates in time on the state temporal sequence. For example, in a graph with highly connected subgraphs, the aggregation~\eqref{eq:aggregation} will weight more the components of such subgraphs which in turn will gain ever greater importance on the state, dominating the other components. \\   
Forgetting factors can mitigate this problem, reducing the influence of past state or inputs. A Gated Graph Neural Network (GGNN)~\cite{ruiz2020gated} is a recurrent Graph Neural Network that uses gating mechanisms to control the information flow in the network. We add two gates, $\bm{\hat{q}}, \bm{\tilde{q}} \in Q \subseteq [0,1]^{N \times F}$, that are multiplied via the Hadamard product~$\circ$ by the state and the inputs of the network, respectively. These two gates regulate how much the past information or the input are used to update the network internal state. GGNNs admit the following state-space representation~\cite{marino2024input},
\begin{equation}
	\begin{cases}
		\tilde{\bm{q}} = \sigma(\tilde{A}_S(\bm{x}) + \tilde{B}_S(\bm{u}) +\hat{b}) \\
		\hat{\bm{q}} = \sigma(\hat{A}_S(\bm{x}) + \hat{B}_S(\bm{u}) +\tilde{b})  \\
		\bm{x}^+ = \sigma_c(\bm{\hat{q}}\circ A_S(\bm{x}) + \bm{\tilde{q}} \circ B_S(\bm{u}) +b)
	\end{cases}
	\label{system}
\end{equation}
with $\sigma(x) = \frac{1}{1+e^{-x}}$ being the logistic function, and $\sigma_c(x) = \frac{e^{x}-e^{-x}}{e^{x}+e^{-x}}$ the hyperbolic tangent.  $\hat{A}_S,\hat{B}_S$ are graph filters of the forgetting gate, $\tilde{A}_S,\tilde{B}_S$ are graph filters~\eqref{eq:filter} of the input gate, and $A_S$ and $B_S$ are the state graph filters~\eqref{eq:filter}. $\bm{\hat{b}},\bm{\tilde{b}},\bm{b} \in \mathbb{R}^{N \times F}$ are respectively the biases of the gates and the state built as $\bm{1}_N \otimes \bm{\textit{b}}$ with the same bias for every agent. We identify the induced $\infty$-norm as $||\cdot||_{\infty}$. We used the following notation for the filters in~\eqref{system}:
\begin{equation}
		\begin{aligned}
			S \triangleq [I, S,\dots, S^K ], & \quad A   \triangleq [A_0, \dots, A_K ]^T, & B  \triangleq [B_0, \dots, B_K ]^T, & \quad \tilde{A} \triangleq [\tilde{A}_0, \dots, \tilde{A}_K]^T, \\ \hat{A}  \triangleq [\hat{A}_0, \dots, \hat{A}_K]^T, & \quad \tilde{B}  \triangleq [\tilde{B}_0, \dots, \tilde{B}_K ]^T, & \hat{B} \triangleq [\hat{B}_0, \dots, \hat{B}_K ]^T
		\end{aligned}
		\label{eq:definitions}
\end{equation}
We consider the system under the following assumption
\begin{assumption} 
	\label{assumption1} 
	The input $\bm{u}$ is unity-bounded: $\bm{u} \in \mathcal{U} \subseteq [-1,1]^{N \times G}$ , i.e. $||\bm{u}||_{\infty} \leq 1$. Given two support matrices $ ||S_{1}(t)||_{\infty},$ $ ||S_{2}(t)||_{\infty}, \forall t \in \mathbb{Z}^+$ associated with two different graphs, they are bounded by the same $ ||\bar{S}||_{\infty}$.
\end{assumption}
\noindent This assumption is quite mild, since the input signal is usually normalized or is the result of other network layers with unitary output activation functions. Assumption~\ref{assumption1} can be met without any further restriction on the graph topology when the support matrix used is normalized, e.g. a normalized Laplacian ($ \bm{Ln} = \bm{D^{-1}L}$). \\
\begin{figure*}
	\centering
	\includegraphics[width=\linewidth]{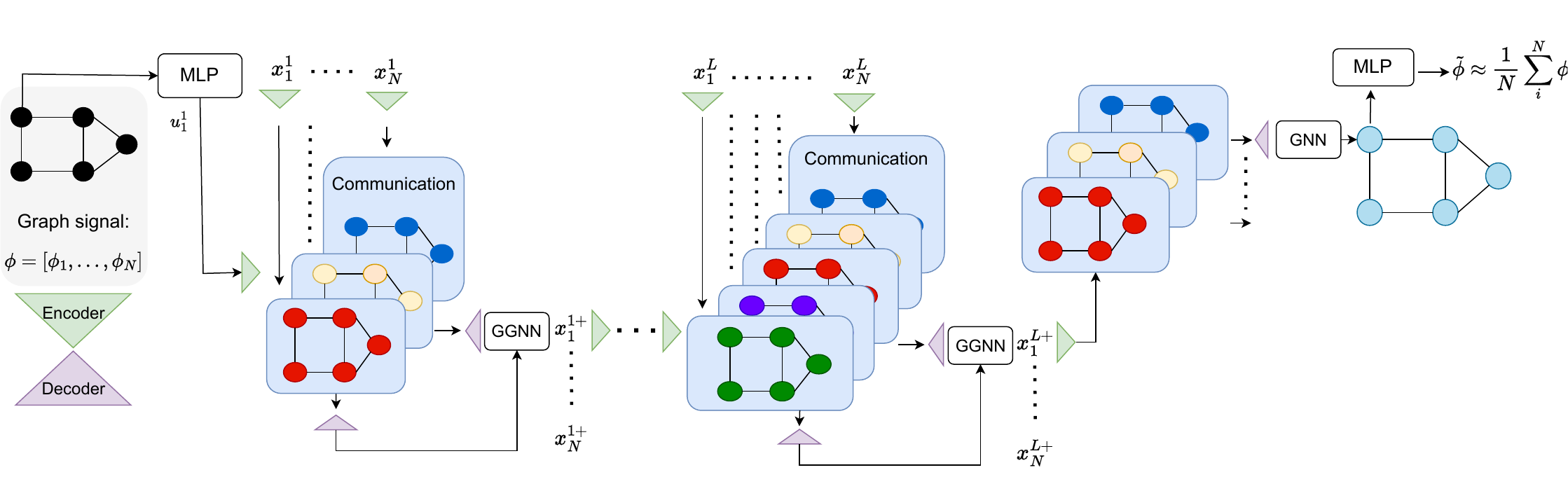}
	\caption{Proposed Graph Neural Network: each node $i$ own its own signal $ \phi_i$ that are fed into a MLP layer and elaborated by $l_n$ layers of gated graph neural layers. The nodes exchange the output features of the previous layers and the layer hidden states features. A final MLP (top right) computes the node average estimation.}
	\label{fig:network_structure}
\end{figure*}
We give a condition on the GGNN weights to achieve $ \delta$ISS.  The $ \delta$ISS property ensures that any pair of state trajectories converge towards each other even if they start from different initial conditions. $ \delta$ISS is defined as:
\begin{definition}[$\delta$ISS]
	\label{dISS_def}
	System~\eqref{system} is called incrementally input-to-state stable~\cite{bayer2013discrete} if there exist functions $\beta_{\delta} \in \mathcal{KL} $ and $\gamma_{\delta} \in \mathcal{K}_\infty$ such that, for any $t \in \mathbb{Z} \geq 0$, any initial states $\bm{x}(0)_1,\bm{x}(0)_2  \in \mathcal{X}$ any input sequences $\bm{u_1},\bm{u_2} \in \mathcal{U}$ it holds that:
	\begin{equation}
				||\bm{x}(t)_1 - \bm{x}(t)_2||_{\infty} \leq \beta_{\delta}(||\bm{x}(0)_1 - \bm{x}(0)_2 ||_{\infty},t) + \gamma_{\delta}(||\bm{u}_1 - \bm{u}_2||_{\infty})
	\end{equation}
\end{definition}
%\begin{remark}
%The use of the $\infty$-norm is more restrictive in terms of signal state amplitude since $||\bm{x}(t)||_{\infty} \leq ||\bm{x}(t)||_2$ but it less restricting for the input tensor which single elements can be larger
%\end{remark}
\begin{remark} 
	In the neural network context, the $\delta$ISS property ensures that any difference in the initial conditions will be eventually discarded, and thus the same outputs will correspond to the same observations. Moreover, since the stability is valid for $t>0$, for a training with a finite time sequence dataset it is guaranteed that all the NN state trajectories converge to a unique solution.
\end{remark}
In this work~\cite{marino2024input}, the author demonstrated
\begin{prop}
	\label{dISS_stab}
	Under Assumptions~\ref{assumption1}, a sufficient condition for the system~\eqref{system} to be $ \delta$ISS is $ \delta \mathcal{A} \leq 1$; where
	\begin{flalign}
				\delta \mathcal{A}  \triangleq \sigma_{\hat{q}}||\bar{S}_K||_\infty ||A||_{\infty}+\frac{1}{4}||\bar{S}_K||^2_\infty||\hat{A}||_{\infty} ||A||_{\infty}  + \frac{1}{4} ||\bar{S}_K||^2_\infty||\tilde{A}||_{\infty} ||B||_{\infty}.
	\end{flalign}
	with
	\begin{flalign}
			\begin{aligned}
				\sigma_{\hat{q}} & \triangleq \sigma( ||S_K||_{\infty}(|| \hat{A} ||_{\infty}+||\hat{B}||_{\infty})+||\hat{\bm{b}}||_\infty).  \\
				\sigma_{\tilde{q}} & \triangleq \sigma( ||S_K||_{\infty}(||\tilde{A}||_{\infty}+|| \tilde{B}||_{\infty}) +||\tilde{\bm{b}}||_\infty). 
			\end{aligned}
	\end{flalign}
\end{prop}
\begin{prop}
	\label{deep-delta-stability}
	deep GGNN architectures composed by $l_n$ layers are $ \delta$ISS if every $ i$-th layer satisfies the Theorem~\ref{dISS_stab}.
\end{prop}
\noindent Moreover, with non-instantaneous communication, the neural network is still $ \delta$ISS under the same conditions~\cite{marino2024input}. The performance at deployment can be different from the one at the training time. However, given a reasonable sampling time for practical applications, like $ 0.01s$, the sequence of support matrices will not present drastic changes as their spectral characteristics are similar, thus stable to graph perturbations~\cite{marino2024input}.

\section{Problem Formulation}
\label{sec:problem_formulation}
Consider a multi-agent network group composed by $ N$ agents forming a communication graph $ \mathcal{G}$. Each agent $ i$ has a local bounded and continuous scalar reference signal $ \phi_i: \mathbb{R} \xrightarrow{} \mathbb{R}$, output of a sensor or another algorithm. In the dynamic average consensus problem, each agent must track the average of the agents signals $ \phi_i$ with the tracking objective given by:
\begin{equation*}
	\bar{\phi}(t) = \frac{1}{N}\sum_{i=1}^N \phi_i (t) 
\end{equation*}
with $ t \in \mathcal{Z}^{+}$ being the time iteration. Since the agents know only their signal, the tracking objective requires communication among the agents. The information sharing can be modeled as a function $ f$ whose iterative application provides the dynamic average estimation:
\begin{equation}
 \tilde{\phi}_i(t+1) = f(\phi_i(t), \tilde{\phi}_i(t), \{ \phi_j(t), \tilde{\phi}_j(t) | j \in \mathcal{N}_i \})
\end{equation}
with $ \tilde{\phi^i}(t)$ the current average estimation of the $ i$-th agent at iteration t. The function $ f$ can be approximated by an artificial neural network that captures the graph structure and its properties, described in the next Section. 

\section{Model}
\label{sec:model_training}
We assume that the underlying graph of the multi-agent system is connected throughout the estimation and, consequently, the average tracking is possible. 

Each agent possesses an instance of the neural network illustrated in Figure~\ref{fig:network_structure}. The network is fed with the agent's local signal. Initially, the input signal is transformed into a new feature space consisting of $ G$ features. This transformation is achieved by employing a multi-layer perceptron (MLP) utilizing a hyperbolic tangent activation function. The purpose of this step is to enable the network to operate in a more suitable space, which is modeled during the training process.
Following the initial processing, the transformed signal undergoes processing through $l_n$ layers of GGNNs with $F$ state features and a filter length of $K$. Moreover, an additional GNN layer with $Fl$ features and a filter length of $Kl$ is included to aggregate the states from the preceding layers. Subsequently, the features represented by $Fl$ are projected back into the original signal space using a final MLP, resembling the architecture used for processing the input. In Figure~\ref{fig:network_structure}, the GGNN states, represented by vector $\bm{x}$, act as a memory for the current average estimation, effectively shaping an autoregressive estimator. GGNNs are well-suited for distributed signal tracking due to their inherently distributed and dynamic nature, aligning closely with established literature solutions to the problem~\cite{kia2019tutorial, george2019robust}.   

As a support matrix, we opted for the normalized Laplacian, which is a commonly used choice in the graph neural network literature \cite{I-Lingfei2022GNNFoundations}. We denote the neural network with normalized Laplacian with \textit{GNN}. To improve the convergence results, we evaluate the use of an attention mechanism. Given the sparsity communication patterns, this mechanism leaves the decision of the link weights to the training process based on the data to communicate. In particular, each weight is updated by this rule  
\begin{equation}
	a_{ij} = \frac{exp(\rho(W[x_i | x_j]))}{\sum_{j \in \mathcal{N}_i} exp(\rho(W[x_i | x_j]))} 
\end{equation}
with $\rho$ being the LeakyRelu function, $W$ learning parameters and $\cdot | \cdot$ the operator to stack two variables in a single vector. The new weights of the support matrix can be computed as
\begin{equation}
		\begin{aligned}
			w_{ij} &= -a_{ij}  \quad \text{for} \quad j \in \mathcal{N}_i \\
			w_{ii} &= 1; 
		\end{aligned}
\end{equation}
In this way, the final support matrix has still zeros row sum while weighting more the nodes that contribute more to the desired estimation. In this solution, that we named \textit{GNNa}, we instantiate one attention mechanism for the state filters ($A_{S},\tilde{A}_S,\hat{A}_S$) and one for the input filters ($B_{S},\tilde{B}_S,\hat{B}_S$). 

One of the main challenges that compromise the use of GNNs for multi-agent applications is the amount of variables that are needed to communicate for obtaining accurate predictions. To understand this fact, we can use the unit-delay communication model~\cite{marino2024input} and communicate unit-time delayed signals to compute the graph filters in~\eqref{eq:filter} in one shot, by sending~${[x_i(t), \hdots, \boldsymbol\ell_i^{K-1}\bm{x}(t-K-1)]}$ for each agent, where $\boldsymbol\ell_i^{K-1}$ stands for $K-1$ repeated application of equation~\eqref{eq:aggregation}. This model allows releasing a new output at each communication iteration but at the cost of more communicated variables. For instance, consider the neural network depicted in Figure~\ref{fig:network_structure}. In that example, each agent transmits a total of ${KG+2KFl_n-KF}$ variables for the GGNN and an extra $KlF$ variables for the last GNN layer. Specifically, agents send $KG$ variables corresponding to the input features for the initial layer in the GGNN. Following that, each agent is obliged to transmit a set of $KF$ state feature variables for every one of the $l_n$ layers, along with an identical count of input variables for $l_n-1$ layers. 

Our solution to this drawback is to introduce an encoding-decoding mechanism before the communication. In this way, every filter becomes 
\begin{equation}
	H_{ed\bm{S}}(\bm{x}) = \sum_{k=0}^K d_{\theta}(\bm{S}^k e_{\theta}(\bm{x})) \bm{H_k},
	\label{eq:filter_at}
\end{equation}
letting ${d_{\theta}(\cdot):\mathbb{R}^{N \times F'} \rightarrow \mathbb{R}^{N \times F}}$ and ${ e_{\theta}(\cdot):\mathbb{R}^{N \times F} \rightarrow \mathbb{R}^{N \times F'}}$ respectively the decoding and encoding functions. In particular, $ F' << F$ is meant to reduce the number of variables to communicate per filter, calculated by replacing $G$ and $F$ with the smaller dimensions $G'$ and $F'$. 

By applying the Proposition~\ref{dISS_stab} to a GGNN without attention and encoding-decoding, we can guarantee stable estimation regardless of initial GGNNs state conditions and connectivity degree while imposing a convergence rate, as shown in Section~\ref{sec:model_training}. However, The introduced attention and encoding-decoding require extending the essential result in the Proposition~\ref{dISS_stab}.  We assume, without loss of generality, that the functions $d_{\theta}$ and $e_{\theta}$ are made of two MLPs with ReLU activation functions, but also other Lipschitz continuous activation functions are valid choices. The infinite norm of $e_{\theta}$ is 
\begin{equation}
		\begin{aligned}
			||e_{\theta}(\bm{x})||_{\infty} & \triangleq || \rho(\rho( \dots \rho(\bm{x}W_{0e} + \bm{b}_0) + \dots)W_{ke} + \bm{b_{ke}}) ||_{\infty}
			\\
			& \leq \Pi_{j=0e}^{ke}|| W_j ||_{\infty}||\bm{x}||_{\infty} + \Pi_{j=1e}^{ke}|| W_j ||_{\infty} ||\bm{b_{j-1}} ||_{\infty} + ||\bm{b_{kd}}||_{\infty} \leq E ||\bm{x}||_{\infty} + Eb
		\end{aligned} 
	\label{eq:inf_encoder}
\end{equation}
The weights $W_{je}$ and the biases $\bm{b}_{je}$ are the parameters of the encoding function to learn. In practice, the dimensions of these weights are task-dependent and are chosen to affect as little as possible the tracking. Similarly, the decoding function norm satisfies
\begin{equation*}
	||d_{\theta}(\bm{x})||_{\infty}  \leq D ||\bm{x}||_{\infty} + Db 
\end{equation*}
Thus, the filter infinite norm becomes  
\begin{equation}
		\begin{aligned}
			||H(\bm{x})||_{\infty} & \leq || [d_{\theta}(e_{\theta}(x)), d_{\theta}(Se_{\theta}(x)), \dots, d_{\theta}(S^K e_{\theta}(x))]||_{\infty}||H||_{\infty} \\
			& \leq [K D ||S_{K}||_\infty  (E ||\bm{x}||_{\infty}+Eb) + Db]||H||_{\infty}
		\end{aligned}
		\label{eq:inf_decoder} 
\end{equation}
with $ H = [H_0,H_1,\dots,H_K]^T$. In the following, we consider two different encoding-decoding pairs to communicate the states and the inputs, denoting the different functions respectively as $e_x,e_u$ and $d_x,d_u$.
\begin{theorem}
	\label{dISS_ed_stab}
	Under assumptions~\ref{assumption1}, a sufficient condition for the system~\eqref{system} with encoding and decoding functions to be $ \delta$ISS is $ \delta\mathcal{A}_{ed} < 1$; where
	\begin{flalign}
			\begin{aligned}
				\delta \mathcal{A}_{ed} & \triangleq K D_x ||\bar{S}_K||_\infty E_x[\sigma_{\hat{q}ed} ||A||_{\infty} + \frac{1}{4}||\hat{A}||_{\infty} ||A||_{\infty}(K D_x ||\bar{S}_K||_{\infty} (E_x  + Eb_x) + Db_x) \\ & + \frac{1}{4} ||\tilde{A}||_{\infty} ||B||_{\infty}(K D_u ||\bar{S}_K||_{\infty}(E_u + Eb_u) + Db_u )].
			\end{aligned}
	\end{flalign}
\end{theorem}
We leave the proof in the Appendix~\ref{proof_theorem1}.

In addition to a direct comparison in experimental use cases reported in Section~\ref{sec:num_results}, we want to provide an intuition on why our architecture helps to reach faster convergence compared to the model-based techniques. 

To facilitate this analysis, we simplify our architecture by considering a single variable and a single unbiased layer of an attention-based GGNN with a graph filter length of 1. Additionally, we omit the principal non-linearity (\textit{tanh}) of the GGNN layer along with the input and output latent space transformations. This results in the following state-space form:
\begin{equation}
		\bm{x}^+ = \bm{\hat{q}}\circ \sum_{k=0}^{1} \bm{S_{A}}^{k}\bm{x} a_k + \bm{\tilde{q}} \circ \sum_{k=0}^{1} \bm{S_{A}}^{k}\bm{u} b_k
	\label{system-wtanh}
\end{equation}
Note that we denote $S_{A}$ the attention-based support matrix. In comparison, we can consider one of the most simple form of model-based average dynamic consensus~\cite{kia2019tutorial}
\begin{equation}
		\bm{x}^+ = \sum_{k=0}^{1} \bm{S}^{k}\bm{x} a_k + \sum_{k=0}^{1}\bm{S}^{k}\bm{u} b_k
	\label{classic-ace}
\end{equation}
where $S$ is the Laplacian matrix. Convergence can be ensured in both cases by setting $a_k < 1 \in \mathbb{R}$. The simplifications render the two systems structurally comparable, but two primary differences remain: the gates and the attention weights over the matrix $S$. With the same $a_k$, gates can dynamically regulate $a_k$ to enhance convergence since $\hat{q}$ changes as a function of the state within the range $[0,1]$. Therefore, in the worst-case scenario, the convergence matches that of~\eqref{classic-ace} in a single iteration. The attention mechanism, on the other hand, acts directly on the weighted graph, causing the neighboring contributions to the state iterations to be uneven compared to a standard Laplacian. This leads to varying convergence speeds across the agents. During training, the attention is optimized to minimize the convergence error in the least number of iterations, thereby reducing the training loss, as demonstrated in the subsequent section.              

\section{Training}
To train and validate the model, we constructed a dataset comprising graphs of various sizes, ranging from 4 to 25 nodes. These graphs exhibited different levels of connectivity and, for each node, we assigned static signals in a bounded set and recorded the corresponding averages. To ensure an adequate level of randomization for the graph connectivity and signal averages, we generated numerous variations of signals and graphs based on Erdős–Rényi model~\cite{erdHos1960evolution}. We bounded the signals between $[-1,1]$, assuming that in real scenarios we can normalize the local signal. This choice is validated numerically in the section~\ref{sec:num_results}. The dataset is available at the following link\footnote{https://tinyurl.com/44xpvce4}.

The dataset was then divided into three sets: training (70\%), validation (20\%), and test (10\%). Throughout the training process, the proposed model receives the training signals and graph Laplacian as inputs, which are repeated for a total of $T$ iterations. At each iteration, the model produces an average estimation and updates its internal state. To train the model, we employed the following cost function: 
\begin{equation}
			J  = \frac{1}{T} \sum_{t}^{T} || \tilde{\bm{\phi}}_t - \bm{1}\bar{\phi} ||_2^2  + \Pi, \quad
			\Pi  = \sum_{i=0}^{L} \text{Softplus}( \delta \mathcal{A}_{i} - 1)
		\label{eq:loss}
\end{equation}
where $\tilde{\bm{\phi}}$ is the vector whose elements are the average estimation of each agent and $\bm{1}$ is a row vector of ones. This loss term serves to guide the agent estimation toward the signal average. The second term $\Pi$ is the regularization term that guides the model to convergence by imposing the condition in the Proposition~\ref{dISS_stab} or Theorem~\ref{dISS_ed_stab} for the GGNNs internal stability with Softplus to have a smooth ReLU function. The Softplus can be regulated by its $\beta$ to enforce a higher contraction rate, $\delta \mathcal{A}_{i}$ below $1$, and consequently a fast convergence. In this way, $\Pi$ assumes the form of a soft constraint with hyperparameter $\beta$ that must be tuned to achieve a desirable convergence rate while keeping a good average estimation error. We chose $\beta = 10$.
\begin{remark}
	To enforce convergence to a specific value of $ \delta\mathcal{A} $, one could formulate the training process as an optimization problem with the hard constraint $ \delta\mathcal{A} < \epsilon < 1 $, as proposed in~\cite{d2022incremental}. However, without a precise understanding of the model's approximation capabilities, choosing a small $ \epsilon $ may significantly degrade the average tracking performance. In contrast, incorporating $ \Pi $ as a soft constraint allows for a balanced trade-off between minimizing the learning error and achieving a desirable convergence rate.

\end{remark}

 By adjusting the number of hidden state features ($F$) and the filter length ($K$), we can enhance the model's accuracy at the expense of increased communication, even if the model with encoding-decoding functions suffers less from this increase. In particular, the number of variables to communicate increases linearly with $F$ and $K$ for what was said in Section~\ref{sec:model_training}. In Tables~\ref{tab:comparison-F} and~\ref{tab:comparison-K}, we compare validation results for three different models: \textit{GNN}, \textit{GNNa}, and \textit{GNNa} with encoding-decoding mechanism (\textit{GNNa-ed}). For the latter, we compress both the state and the input features into $2$ features each thanks to the encoding function. The validation results indicate that a favorable compromise between accuracy and communication is achieved with $F=25$ and $K=2$ with ${\textit{GNNa-ed}}$ model. Notably, the best outcomes are obtained when using \textit{GNNa}, even if {\textit{GNNa-ed}} does show quite similar results. However, according to what already said in Section~\ref{sec:model_training}, \textit{GNNa} with $F=25$ communicates $250$ variables while \textit{GNNa-ed} with $F=25$ and encoding to $2$ features communicates $20$ variables. It is important to note that these results pertain to static graphs and signals.

Moreover, the same network design can be employed to estimate the average of multiple signals by increasing the input and output signal dimensions. This entails feeding the network with more than one signal simultaneously. Consequently, the network capabilities can be expanded while maintaining the same communication requirements as a single average estimation, albeit with relatively lower accuracy, as indicated in Table~\ref{tab:comparison-G}. The Regularization loss is in the range of $[1.27e-5,1.54e-5]$ corresponding to $\delta A \approx 0.1$.
\begin{table}[t]
	\begin{subtable}[t]{0.48\textwidth}
		\resizebox{\textwidth}{0.15\textwidth}{
			\begin{tabular}{l|lll}
				F  & $ GNN$ & $ GNNa$ & $ GNNa-ed$ \\ \hline
				16  & \cellcolor[HTML]{FFCCC9}1.8e-3+1.54e-5 & \cellcolor[HTML]{FFCCC9}0.83e-3+1.23e-5 & \cellcolor[HTML]{FFCCC9}0.84e-3 +1.34e-5      \\
				25 & \cellcolor[HTML]{FFCCC9}1.5e-3+1.24e-5 & \cellcolor[HTML]{DCFFA6}\textbf{0.32e-3}+1.21e-5 & \cellcolor[HTML]{DCFFA6}0.34e-3+1.30e-5         \\
				32 & \cellcolor[HTML]{FFCCC9}1.3e-3+1.45e-5 & \cellcolor[HTML]{DCFFA6}\textbf{0.28e-3}+1.27e-5 & \cellcolor[HTML]{DCFFA6}0.40e-3+1.31e-5
			\end{tabular}
		}
		\caption{}
		\label{tab:comparison-F}
	\end{subtable} 
	\begin{subtable}[t]{0.48\textwidth}
		\resizebox{\textwidth}{0.15\textwidth}{
			\begin{tabular}{l|lll}
				K  & $ GNN$ & $ GNNa$ & $ GNNa-ed$ \\ \hline
				1  & \cellcolor[HTML]{FFCCC9}1.9e-3+1.40e-5 & \cellcolor[HTML]{FFCCC9}0.95e-3+1.32e-5 & \cellcolor[HTML]{FFCCC9}0.99e-3+1.35e-5 \\
				2  & \cellcolor[HTML]{FFCCC9}1.5e-3+1.31e-5 & \cellcolor[HTML]{DCFFA6}\textbf{0.32e-3}+1.39e-5 & \cellcolor[HTML]{DCFFA6}{\color[HTML]{333333}0.34e-3+1.38e-5} \\
				3  & \cellcolor[HTML]{FFCCC9}{\color[HTML]{333333}1.4e-3+1.27e-5} & \cellcolor[HTML]{DCFFA6}\textbf{0.30e-3}+1.29e-5 & \cellcolor[HTML]{DCFFA6}0.33e-3+1.31e-5 
			\end{tabular}
		}
		\caption{}
		\label{tab:comparison-K}
	\end{subtable}
	\begin{subtable}[t]{0.48\textwidth}
		\resizebox{\textwidth}{0.15\textwidth}{
			\begin{tabular}{l|lll}
				G  & $ GNN$ & $ GNNa$  & $ GNNa-ed$  \\ \hline
				1  & \cellcolor[HTML]{FFCCC9}1.5e-3+1.33e-5 & \cellcolor[HTML]{DCFFA6}\textbf{0.32e-3}+1.28e-5 & \cellcolor[HTML]{DCFFA6}{\color[HTML]{333333}0.34e-3+1.30e-5} \\
				4  & \cellcolor[HTML]{FFCCC9}1.9e-3+1.31e-5 & \cellcolor[HTML]{DCFFA6}\textbf{0.45e-3}+1.25e-5 & \cellcolor[HTML]{DCFFA6}{\color[HTML]{333333}0.49e-3+1.28e-5} \\
				10 & \cellcolor[HTML]{FFCCC9}{\color[HTML]{333333}2.7e-3+1.32e-5} & \cellcolor[HTML]{FFCCC9}0.92e-3+1.37e-5 & \cellcolor[HTML]{FFCCC9}1.04e-3+1.28e-5    
			\end{tabular}
		}
		\caption{}
		\label{tab:comparison-G}
		
	\end{subtable}
	\caption{Validation + regulation error computed using eq~\eqref{eq:loss} given by training the neural network in Figure~\ref{fig:network_structure} with $ 2$ GGNN layers using left normalized Laplacian \textit{GNN}, with attention mechanism \textit{GNNa} and adding encoding-decoding \textit{GGNa-ed}, varying hidden features $ F$ fixing the filter length $K=2$, varying $K$ with $F=25$ and different input size ($ G$) with $K=2,F=25$. We highlighted in green the preferable solutions and in bold the best results.}
\end{table}
\section{Results}
\label{sec:num_results}
The learned neural network was tested in three different simulated scenarios: high-frequency signals (Sect~\ref{sec:high_frequencies_signals}), connectivity maintenance control (Sect~\ref{sec:connectivity_maintenance}), and formation control (Sect.~\ref{sec:formation_control}). These diverse settings allowed us to evaluate the transferability of the neural network across different scales, assess its performance in dynamic graph situations. For the first scenario, we numerically evaluate our apporach using random static graphs. For the other two scenarios, we conducted simulation experiments under non-instantaneous communication deployed with ROS to demonstrate the robustness of our approach in this aspect. In particular, we employed the already mentioned unit-delay communication model with ensured communication at T=$0.01s$. Moreover, we compare the approach proposed against three dynamic average estimators as baselines. The first one is the discrete \textit{PI-ACE}~\cite{kia2019tutorial}, commonly used in distributed control, which requires communication of two variables for estimation. The second one is the robust average estimator \textit{R-ACE}~\cite{george2019robust}, which is independent of state initialization under dynamic graphs. The third one is the more recent method prescribed-performance ACE ( \textit{PP-ACE})~\cite{stamouli2019robust}, designed to achieve fast estimation with high-frequency signals, requiring communication of only one variable for each average estimation. All three algorithms ensure asymptotic convergence under bounded signals with bounded derivatives.

\subsection{Hardware and Experimental Details}
\begin{figure*}[t]
	\begin{subfigure}[b]{0.45\textwidth}
		\centering
		\includegraphics[height=0.6\textwidth,width=\textwidth]{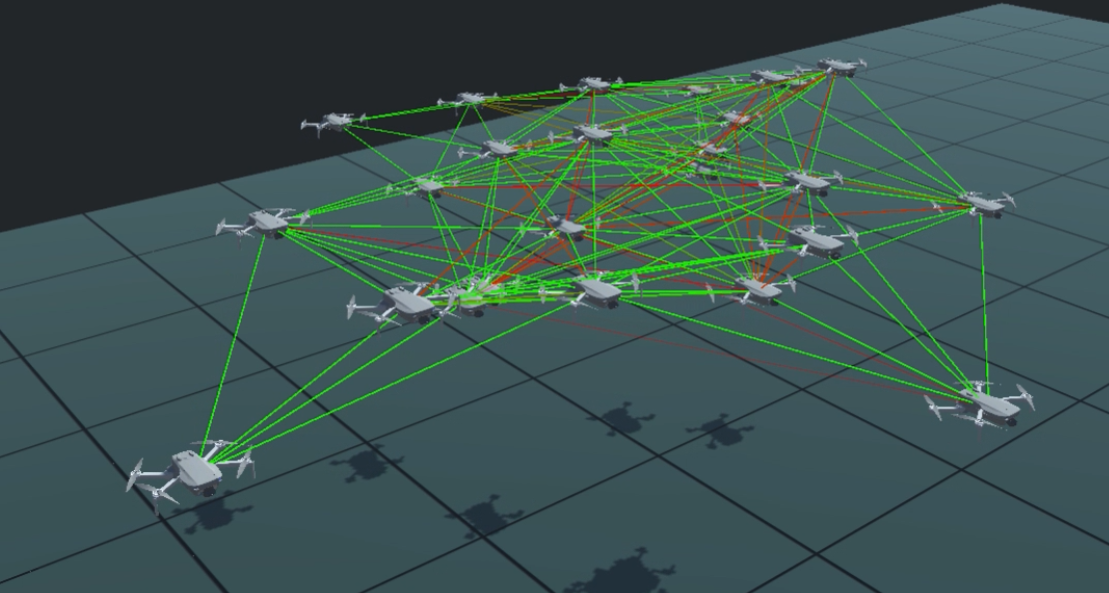}
		\caption{connectivity mantainance}
	\end{subfigure}
	\begin{subfigure}[b]{0.45\textwidth}
		\centering
		\includegraphics[height=0.6\textwidth,width=\textwidth]{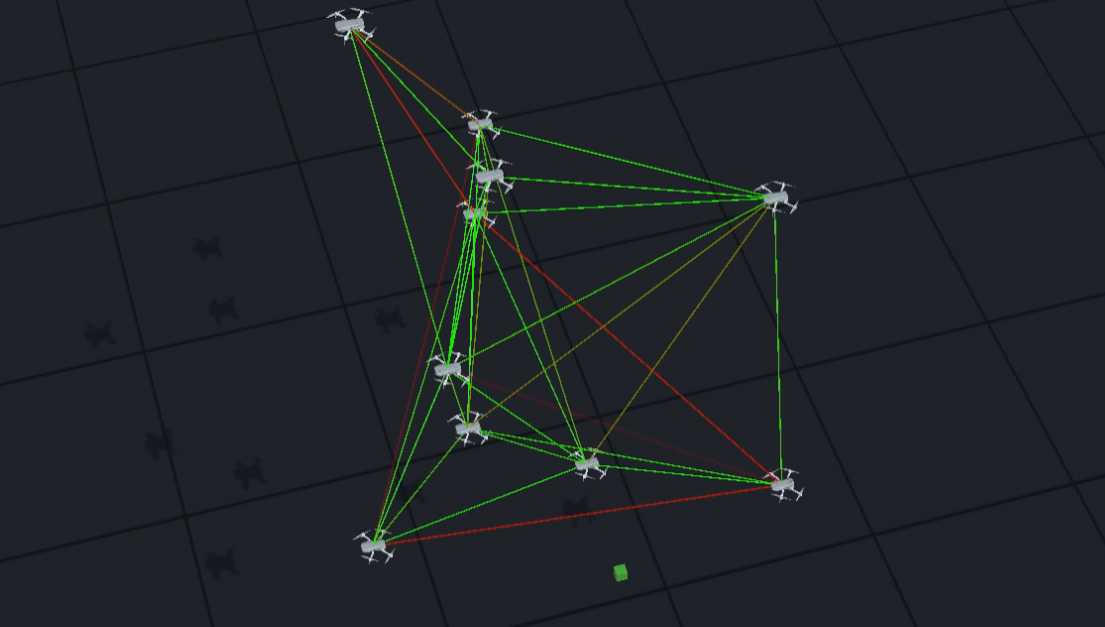}
		\caption{formation control}
	\end{subfigure}
	\caption{Test scenarios for connectivity mantainance (on the left) and formation control (on the right).}
	\label{fig:scenarios}
\end{figure*}

We train our neural model using PyTorch on a server running Ubuntu 22.04. with Intel Core i7-9750H @ 2.60GHz CPU, Nvidia RTX 2080Ti and 32G RAM. For testing the two scenarios of connectivity mantainance and formation control in figure~\ref{fig:scenarios}, we leveraged ROS for the communication among the agents with at a frequency of $100Hz$. We simulate the creation of a dynamic communication graph by employing a virtual proximity sensor which allows to select only the information of the agents close at a distance $\leq R$. Moreover, we also simulate the agent dynamics and the control strategy at $100Hz$.

\subsection{Static Signals}
\label{sec:static_signals}
\begin{figure*}[t]
	\begin{subfigure}[b]{0.325\textwidth}
		\centering
		\includegraphics[width=\textwidth, height=0.7\textwidth]{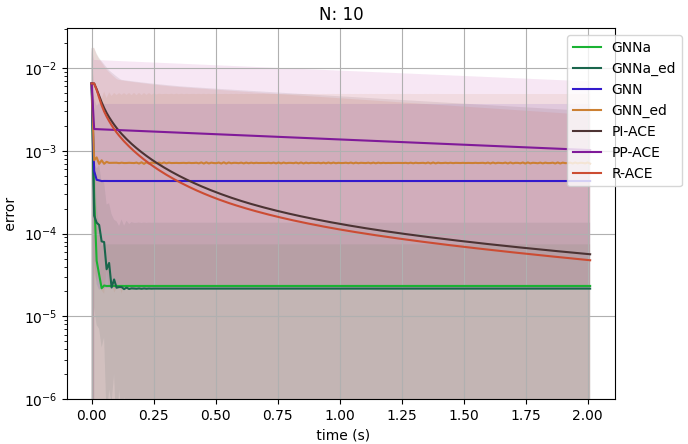}
		\caption{error for graph size of N=10}
		\label{fig:signals_10}
	\end{subfigure}
	\begin{subfigure}[b]{0.325\textwidth}
		\centering
		\includegraphics[width=\textwidth,height=0.7\textwidth]{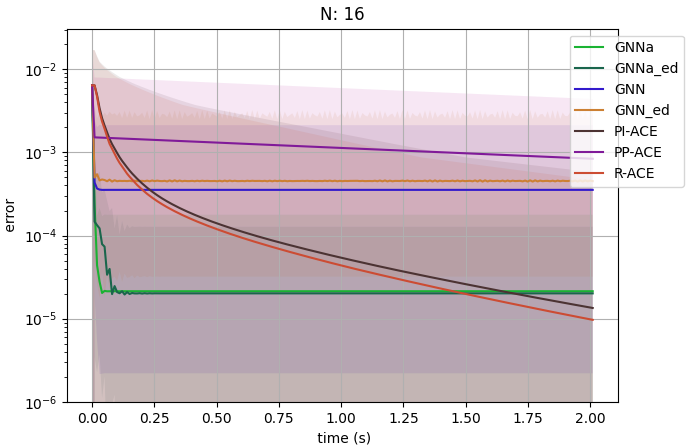}
		\caption{error for graph size of N=16}
		\label{fig:signals_16}
	\end{subfigure}
	\begin{subfigure}[b]{0.325\textwidth}
		\centering
		\includegraphics[width=\textwidth,height=0.7\textwidth]{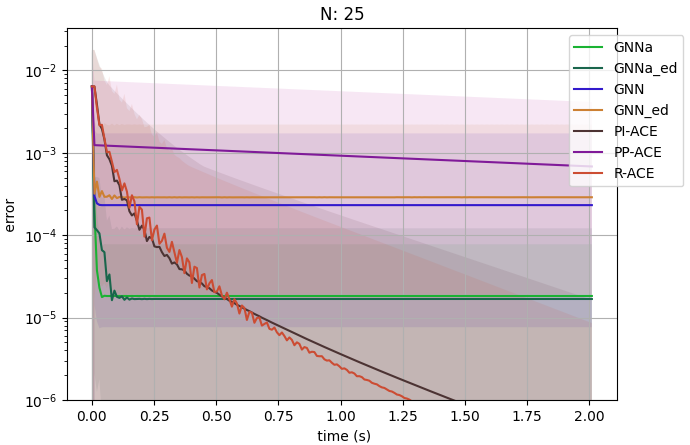}
		\caption{error for graph size of N=25}
		\label{fig:signals_25}
	\end{subfigure}
	\caption{\textbf{Static signals}. Mean and standard deviation error of the average estimation on $1000$ graphs and static signals for (a) $10$ agents, (b) $16$ agents, and (c) $25$ agents, applying GNN with attention mechanism (\textit{GNNa}), GNNa adding encoding-decoding (\textit{GNNa-ed}), GNN with left-normalized Laplacian (\textit{GNN}), \textit{GNN} adding encoding-decoding (\textit{GNN-ed}), \textit{PI-ACE}, \textit{PP-ACE}, and \textit{R-ACE}.}
	\label{fig:static_signals}
\end{figure*}

For this set of experiments, we evaluate the models on 1000 randomly generated graphs using static signals. The graphs contain $N=[10,16,25]$ nodes and are generated using the Erdős-Rényi method, consistent with the training distribution. As illustrated in Figure~\ref{fig:static_signals}, the neural network with attention mechanism (\textit{GNNa}) and its encoding-decoding variant (\textit{GNNa-ed}) achieves comparable errors around $\num{2e-5}$ the lowest tracking error across all cases, demonstrating strong generalization to different graph sizes. However, \textit{GNNa-ed}) results marginally slower to converge. We believe this is due to the compromise between accuracy and convergence imposed by the loss function~\eqref{eq:loss}. \textit{GNN}, without attention, exhibits significantly higher errors and variance, indicating sensitivity to graph structure and limited adaptability to varying node cardinalities. This behaviour confirms that the attention mechanisms play a crucial role in enhancing robustness against distribution shifts. \textit{GNN-ed} reduces communication costs at the expense of accuracy, resulting in an error ${\num{2e-4}}$ higher than \textit{GNN}. Regarding model-based approaches, \textit{PI-ACE} and \textit{R-ACE} exhibit asymptotic convergence to zero error, confirming their ability to track static signals given sufficient time. However, they show also strong dependency on the graph sizes and structure as confirmed by their high standard deviations. In contrast, \textit{PP-ACE} fails to provide accurate estimations since it is only effective for signals with bounded derivatives.

\subsection{High Frequencies Signals}
\label{sec:high_frequencies_signals}
\begin{figure*}[t]
	\begin{subfigure}[b]{0.325\textwidth}
		\centering
		\includegraphics[width=\textwidth, height=0.7\textwidth]{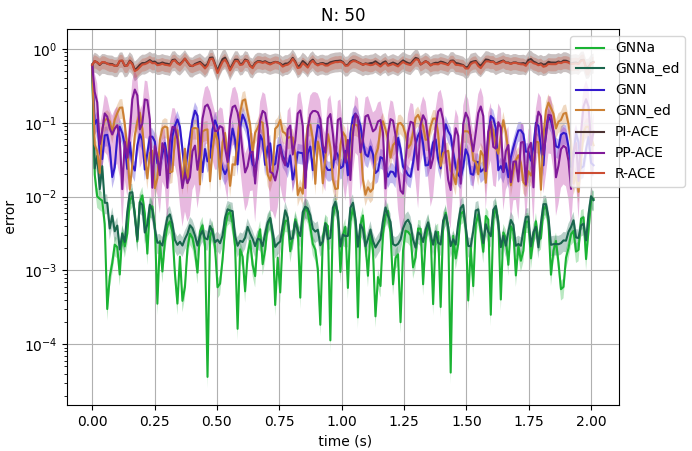}
		\caption{error for graph size of N=50}
		\label{fig:20hz_signals_50}
	\end{subfigure}
	\begin{subfigure}[b]{0.325\textwidth}
		\centering
		\includegraphics[width=\textwidth,height=0.7\textwidth]{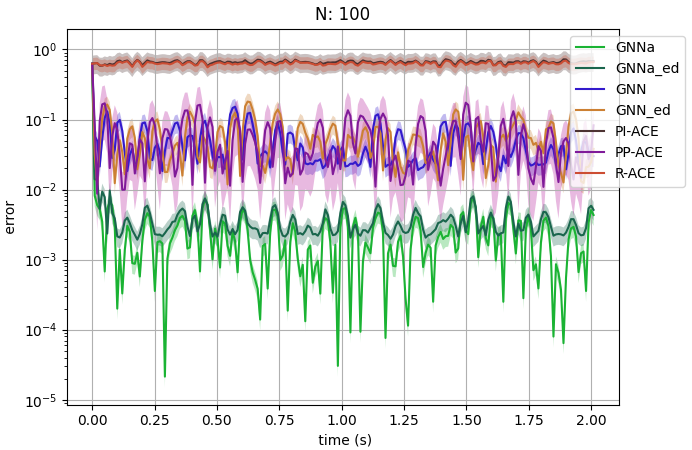}
		\caption{error for graph size of N=100}
		\label{fig:20hz_signals_100}
	\end{subfigure}
	\begin{subfigure}[b]{0.325\textwidth}
		\centering
		\includegraphics[width=\textwidth,height=0.7\textwidth]{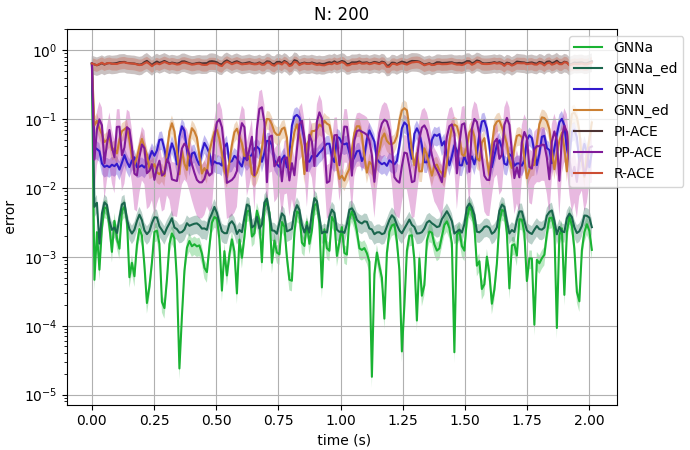}
		\caption{error for graph size of N=200}
		\label{fig:20hz_signals_200}
	\end{subfigure}
	\caption{\textbf{High frequencies signals}. Mean and standard deviation error of the average estimation on $1000$ graphs and $20$Hz signals for (a) $50$ agents, (b) $100$ agents, and (c) $200$ agents, applying GNN with attention mechanism (\textit{GNNa}), GNNa adding encoding-decoding (\textit{GNNa-ed}), GNN with left-normalized Laplacian (\textit{GNN}), \textit{GNN} adding encoding-decoding (\textit{GNN-ed}), \textit{PI-ACE}, \textit{PP-ACE}, and \textit{R-ACE}.}
	\label{fig:20hz_signals}
\end{figure*}
\begin{figure*}[t]
	\begin{subfigure}[b]{0.325\textwidth}
		\centering
		\includegraphics[width=\textwidth]{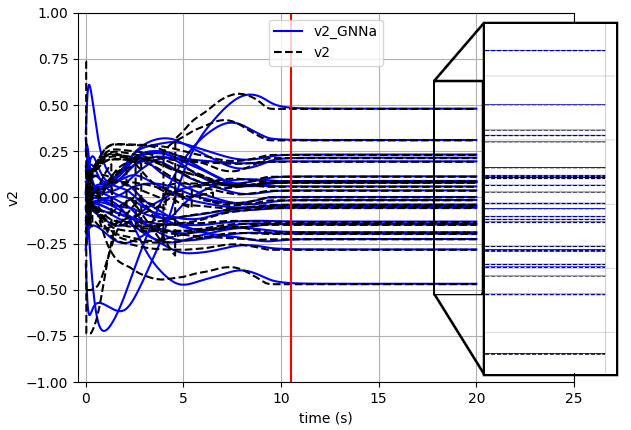}
		\caption{$v_2$ tracking vector with \textit{GGNa} model}
		\label{fig:connectivity_La}
	\end{subfigure}
	\begin{subfigure}[b]{0.325\textwidth}
		\centering
		\includegraphics[width=\textwidth]{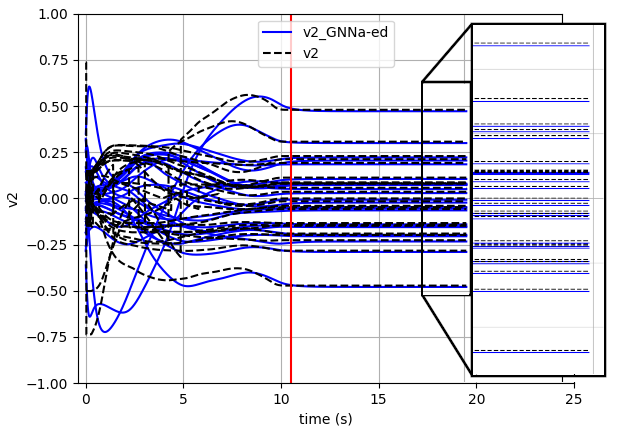}
		\caption{$v_2$ tracking vector with \textit{GGNa-ed} model}
		\label{fig:connectivity_Ln}
	\end{subfigure}
	\begin{subfigure}[b]{0.325\textwidth}
		\centering
		\includegraphics[width=\textwidth]{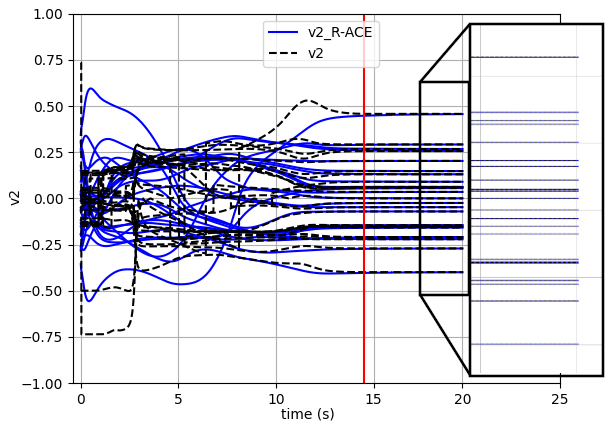}
		\caption{$v_2$ tracking vector with \textit{R-ACE} algorithm}
		\label{fig:connectivity_Pi}
	\end{subfigure}
	\caption{\textbf{Connectivity Maintenance}. Laplacian second eigenvector $v2$ estimated during the connectivity control thanks to the average estimation, using (a) \textit{GNN} with attention mechanism (\textit{GNNa}), (b) \textit{GNNa} adding encoding-decoding mechanism (\textit{GNNa-ed}),  and (c) \textit{R-ACE}.}
	\label{fig:connectivity_control}
\end{figure*}
\begin{figure*}[t]
	\begin{subfigure}[b]{0.325\textwidth}
		\centering
		\includegraphics[width=\textwidth]{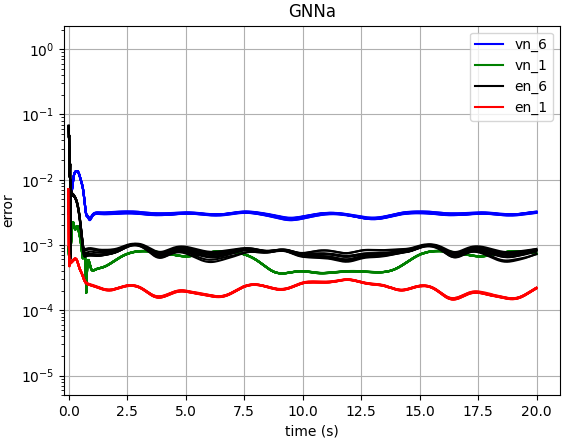}
		\caption{tracking error for \textit{GNNa}}
		\label{fig:formation_a}
	\end{subfigure}
	\begin{subfigure}[b]{0.325\textwidth}
		\centering
		\includegraphics[width=\textwidth]{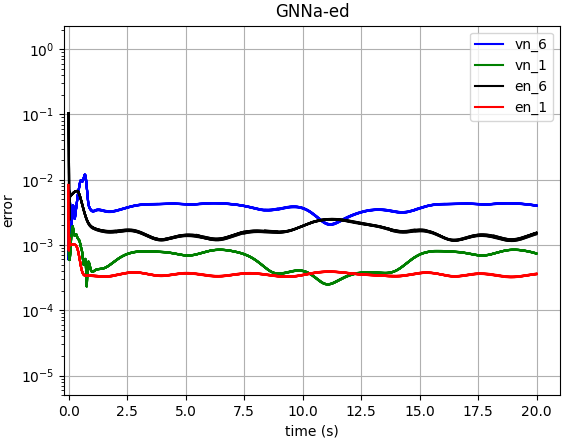}
		\caption{tracking error for \textit{GNNa-ed}}
		\label{fig:formation_n}
	\end{subfigure}
	\begin{subfigure}[b]{0.325\textwidth}
		\centering
		\includegraphics[width=\textwidth]{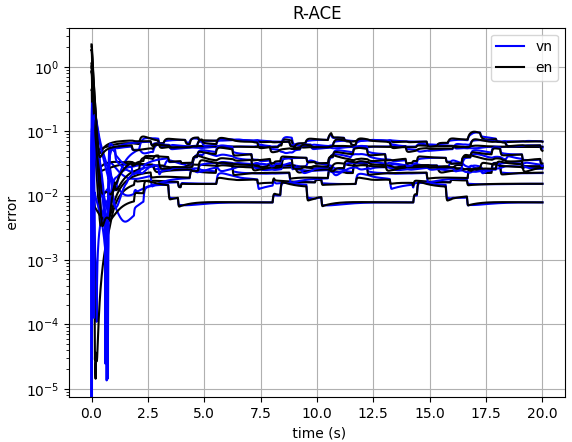}
		\caption{tracking error for \textit{R-ACE}}
		\label{fig:formation_pi_ace}
	\end{subfigure}
	\caption{\textbf{Formation Control}. Error in logarithmic scale for the estimation of the average velocity $vn$ and the average agent-target distance $en$. The figures show the errors with 6-dimensional and 1-dimensional input features by using (a) \textit{GNN} with attention mechanism (\textit{GNNa}), (b) \textit{GNNa} adding encoding-decoding mechanism (\textit{GNNa-ed}),  and (c) \textit{R-ACE}.}
	\label{fig:formation-control}
\end{figure*}
For this set of experiments, we assign $20$Hz sinusoidal signals to the agents, with amplitudes ranging between $[-2,2]$ and phases within $[-\pi, \pi]$. We evaluate the average performance by testing all methods on 1000 randomly generated graphs of sizes $[50,100,200]$ using the Barabási-Albert method, ensuring a different graph distribution from training, where we use the Erdős-Rényi method. As shown in Figure~\ref{fig:20hz_signals}, the neural network with attention mechanism (\textit{GNNa}) and encoding-decoding (\textit{GGNa-ed}) achieve the best tracking performance among all tested methods, with \textit{GNNa} exhibiting the lowest tracking error. Specifically, \textit{GNNa} performs well with an average error around ${\num{1.5e-3}}$, reaching a minimum error of ${\num{2e-5}}$ across different graph sizes. The primary distinction between \textit{GNNa} and \textit{GGNa-ed} is that \textit{GGNa-ed} maintains a more stable error around ${\num{3.5e-3}}$, whereas \textit{GNNa} has slightly higher consensus error despite achieving a lower average tracking error. This difference becomes more evident in subsequent examples. \textit{GNN} performs comparably to \textit{PP-ACE}, with an average error of $5e-2$ and a high standard deviation. \textit{PP-ACE} attempts to impose prescribed transient performance using a performance function, which can make the estimation numerically unstable, leading to challenges in tuning the function. Additionally, its performance depends on the signal derivatives, making \textit{PP-ACE} less effective for static or slow-varying signals. \textit{GNN} with encoding-decoding (\textit{GNN-ed}) achieves a similar error to a standard \textit{GNN}, with the advantage of communicating fewer variables due to the encoder. As expected, \textit{PI-ACE} and \textit{R-ACE} exhibit the worst performance, as they are not designed to handle high-frequency signals and varying graph sizes.

\subsection{Connectivity Maintenance}   
\label{sec:connectivity_maintenance} We used the algorithm in~\cite{C-Freeman2010ConnectivityEstimation} to estimate the connectivity of a group of agents moving with single integrator dynamics and controlled with the algorithm in~\cite{robuffo2013passivity} to enforce their general connectivity. The connectivity estimation algorithm requires communication while the control acts based on the actual knowledge of the connectivity without further communication. The algorithm depicted in~\cite{C-Freeman2010ConnectivityEstimation} iteratively estimates the Laplacian second eigenvector ($v_2$) and its eigenvalue that measures the connectivity degree of the graph by exploiting two average estimations. In this respect, we tested \textit{GNNa-ed}, \textit{GNNa} and \textit{R-ACE} by initializing with the same settings the connectivity estimation and posing $25$ agents in the same positions, in order to isolate the different contributions of the three estimations. Note that we instantiate twice the estimators since they work with a mono-dimensional input. For the sake of brevity, we show here only the \textit{R-ACE}, since it has a better performance compared to \textit{PI-ACE} and \textit{PP-ACE}. We reported the results of these latter in the Appendix. In Figure \ref{fig:connectivity_control}, we show the resulting tracking of $v_2$ while the team uses its estimation to increase the connectivity. As we can see, the \textit{GNN}-based approaches converge $3.0s$ earlier than \textit{R-ACE}, even if they do not show asymptotically convergence as \textit{R-ACE}. However, we can say that they achieve close-to-zero error, with a mean steady-state tracking error of \num{2.5e-2} for {\textit{GNNa}} and \num{4.1e-2} for \textit{GNNa-ed}. Faster convergence is due to the regularization term in~\eqref{eq:loss} which enforces the state fastest convergence possible for an accepted low training error. 
\subsection{Formation Control}
\label{sec:formation_control} Finally, we evaluated the average estimation performances under multiple-dimensional signals in the case of formation control. We considered $10$ double-integrator agents moving freely in space and with linear acceleration inputs $\bm{u}$ saturated in the range of $[-10, \hspace{0.1cm} 10 ] \hspace{0.1cm} m/s^2 $. The formation is controlled to be centred above a moving target to track. Additionally, each agent follows a prescribed velocity pattern $\bm{v_{ti}}$. A centralized control for the agent $i$ can be formulated as:
\begin{equation}
 \bm{u_i} = (\bm{v_{ti}} - \bm{v_i}) - \frac{K_v}{N}\bm{1}\bm{v} - \frac{K_p}{N} \bm{1}\bm{e} 
	\label{eq:centralized-control}
\end{equation}
where $K_v, K_p > 0 $ are controller gains, $\bm{e}$ are the relative positions of the target from the agents and $\bm{v}$ are the agents velocities. We used $K_v=10$, $K_p=14$. $\bm{e}$ and $\bm{v}$ averages in~\eqref{eq:centralized-control} can be computed in a decentralized way by dynamic average estimation of $6$-D vector formed by stacking the agent $3$-D velocity and $3$-D target-agent distance. The vector is entirely delivered to the neural networks trained for an input size of $6$ and an output giving the average for each input vector entry. 

We tested the proposed approach with a circular target trajectory around the origin, with a radius of $5m$ and $0.5 rad/s$, angular frequency, with a $\bm{v_{ti}}$ a planar circle with random magnitudes and phases different for each agent. We compared the average $\bm{v}$ and the average $\bm{e}$ estimation error resulting from \textit{GNNa}, \textit{GNNa-ed} and the \textit{R-ACE}. We also enclosed \textit{PI-ACE} and \textit{PP-ACE} results in the Appendix. We reported the performances of $6$-D and $1$-D input neural networks. \textit{GNNa} estimation error results the lowest among the three methods even if similar to \textit{GNNa-ed}. The $6$-D neural networks perform worst in terms of accuracy than $1$-D \textit{GNNa} and \textit{GNNa-ed} with an error of about an order of magnitude more. The increased error was expected since more information must be encoded in the same number of communicated variables. However, \textit{R-ACE} shows the biggest error compared to $1$-D and $6$-D neural networks, attesting between \num{7e-3} and \num{8.2e-3}. In particular, \textit{R-ACE} slow convergence is visible at the beginning of the experiment where the agents are far from the target. As we anticipated in the Section~\ref{sec:high_frequencies_signals}, \textit{GNNa} has a higher consensus error than \textit{GNNa-ed} which is particularly visible on the $6$-D $\bm{e}$. In the \textit{GNNa-ed} case, the attention mechanism is computed on the encoded features which are the real variables to communicate. Since these latter are less numerous than \textit{GNNa} case, the training process is able to optimize better the support matrix weights thanks to the attention mechanism. Additionally, \textit{R-ACE} shows a larger consensus error than the GNNs based methods. The smoothing effect of graph filters explains this last fact. 
\section{Conclusions}
\label{conclusions}
This paper introduces a novel approach using a gated graph neural network (GNN) to estimate the dynamic average of signals distributed among a group of agents. Traditional methods often struggle with slow convergence rates and limited scalability when dealing with large graphs or high-speed signal changes. Our approach mitigates these issues by integrating a stability condition into the training process of the neural network, thereby enhancing the convergence rate. This integration allows the system to maintain high performance across various scenarios, including those involving rapidly changing signals and complex network topologies. By direct comparisons, we demonstrated that our approach outperforms the baselines in terms of both convergence rate and, in certain cases, precision for high-speed signals, connectivity, and formation control, regardless of the size of the graph. One notable limitation of our approach is the increased communication overhead. The method requires a relatively high number of variables to be exchanged among agents, which can be seen as a drawback compared to more communication-efficient techniques. Despite this, the superior performance in terms of convergence rate and precision makes our method a promising candidate for further research. Future work will focus on optimizing the communication requirements, aiming to maintain the performance benefits while reducing the associated overhead. This optimization will be crucial for practical applications where communication bandwidth is limited or expensive.

\begin{acks}
This work was supported by the ANR-20-CHIA-0017 project ``MULTISHARED''.
\end{acks}

\bibliographystyle{ACM-Reference-Format}
\bibliography{bibliography}

@article{ruiz2020gated,
  title={Gated graph recurrent neural networks},
  author={Ruiz, Luana and Gama, Fernando and Ribeiro, Alejandro},
  journal={IEEE Transactions on Signal Processing},
  volume={68},
  pages={6303--6318},
  year={2020},
  
}

@inproceedings{stamouli2019robust,
  title={Robust dynamic average consensus with prescribed performance},
  author={Stamouli, Charis J and Bechlioulis, Charalampos P and Kyriakopoulos, Kostas J},
  booktitle={2019 IEEE 58th Conference on Decision and Control (CDC)},
  pages={5420--5425},
  year={2019},
  
}

@article{george2019robust,
  title={Robust dynamic average consensus algorithms},
  author={George, Jemin and Freeman, Randy A},
  journal={IEEE Transactions on Automatic Control},
  volume={64},
  number={11},
  pages={4615--4622},
  year={2019},
  
}

@inproceedings{iancu2021towards,
  title={Towards finite-time consensus with graph convolutional neural networks},
  author={Iancu, Bianca and Isufi, Elvin},
  booktitle={2020 28th European Signal Processing Conference (EUSIPCO)},
  pages={2145--2149},
  year={2021},
  
}

@article{porfiri2007tracking,
  title={Tracking and formation control of multiple autonomous agents: A two-level consensus approach},
  author={Porfiri, Maurizio and Roberson, D Gray and Stilwell, Daniel J},
  journal={Automatica},
  volume={43},
  number={8},
  pages={1318--1328},
  year={2007},
  publisher={Elsevier}
}

@inproceedings{olfati2005distributed,
  title={Distributed Kalman filter with embedded consensus filters},
  author={Olfati-Saber, Reza},
  booktitle={Proceedings of the 44th IEEE Conference on Decision and Control},
  pages={8179--8184},
  year={2005},
  
}

@inproceedings{sandryhaila2014finite,
  title={Finite-time distributed consensus through graph filters},
  author={Sandryhaila, Aliaksei and Kar, Soummya and Moura, Jos{\'e} MF},
  booktitle={2014 IEEE International Conference on Acoustics, Speech and Signal Processing (ICASSP)},
  pages={1080--1084},
  year={2014},
  
}

@inproceedings{freeman2006stability,
  title={Stability and convergence properties of dynamic average consensus estimators},
  author={Freeman, Randy A and Yang, Peng and Lynch, Kevin M},
  booktitle={Proceedings of the 45th IEEE Conference on Decision and Control},
  pages={338--343},
  year={2006},
  
}

@inproceedings{spanos2005dynamic,
  title={Dynamic consensus on mobile networks},
  author={Spanos, Demetri P and Olfati-Saber, Reza and Murray, Richard M},
  booktitle={IFAC world congress},
  pages={1--6},
  year={2005}
}

@article{xu2021robust,
  title={Robust finite-time dynamic average consensus with exponential convergence rates},
  author={Xu, Kedong and Gao, Lan and Chen, Fei and Li, Chaojie and Xuan, Qi},
  journal={IEEE Transactions on Circuits and Systems II: Express Briefs},
  volume={68},
  number={7},
  pages={2578--2582},
  year={2021},
  
}

@article{nosrati2012dynamic,
  title={Dynamic average consensus via nonlinear protocols},
  author={Nosrati, Shahram and Shafiee, Masoud and Menhaj, Mohammad Bagher},
  journal={Automatica},
  volume={48},
  number={9},
  pages={2262--2270},
  year={2012},
  publisher={Elsevier}
}

@inproceedings{chen2013robust,
  title={Robust distributed average tracking for coupled general linear systems},
  author={Chen, Fei and Ren, Wei},
  booktitle={Proceedings of the 32nd Chinese Control Conference},
  pages={6953--6958},
  year={2013},
  
}

@article{chen2012distributed,
  title={Distributed average tracking of multiple time-varying reference signals with bounded derivatives},
  author={Chen, Fei and Cao, Yongcan and Ren, Wei},
  journal={IEEE Transactions on Automatic Control},
  volume={57},
  number={12},
  pages={3169--3174},
  year={2012},
  
}

@article{d2022incremental,
  title={An incremental input-to-state stability condition for a generic class of recurrent neural networks},
  author={D'Amico, William and La Bella, Alessio and Farina, Marcello},
  journal={arXiv preprint arXiv:2210.09721},
  year={2022}
}

@inproceedings{wang2011control,
  title={A control perspective for centralized and distributed convex optimization},
  author={Wang, Jing and Elia, Nicola},
  booktitle={2011 50th IEEE conference on decision and control and European control conference},
  pages={3800--3805},
  year={2011},
  
}

@inproceedings{listmann2009consensus,
  title={Consensus for formation control of nonholonomic mobile robots},
  author={Listmann, Kim D and Masalawala, Mohanish V and Adamy, Jurgen},
  booktitle={2009 IEEE international conference on robotics and automation},
  pages={3886--3891},
  year={2009},
  
}

@article{zhu2011distributed,
  title={On distributed convex optimization under inequality and equality constraints},
  author={Zhu, Minghui and Martinez, Sonia},
  journal={IEEE Transactions on Automatic Control},
  volume={57},
  number={1},
  pages={151--164},
  year={2011},
  
}

@article{wang2017diffusion,
  title={Diffusion distributed Kalman filter over sensor networks without exchanging raw measurements},
  author={Wang, Guoqing and Li, Ning and Zhang, Yonggang},
  journal={Signal Processing},
  volume={132},
  pages={1--7},
  year={2017},
  publisher={Elsevier}
}

@article{kia2019tutorial,
  title={Tutorial on dynamic average consensus: The problem, its applications, and the algorithms},
  author={Kia, Solmaz S and Van Scoy, Bryan and Cortes, Jorge and Freeman, Randy A and Lynch, Kevin M and Martinez, Sonia},
  journal={IEEE Control Systems Magazine},
  volume={39},
  number={3},
  pages={40--72},
  year={2019},
  
}

@article{erdHos1960evolution,
  title={On the evolution of random graphs},
  author={Erd{\H{o}}s, Paul and R{\'e}nyi, Alfr{\'e}d and others},
  journal={Publ. math. inst. hung. acad. sci},
  volume={5},
  number={1},
  pages={17--60},
  year={1960}
}

@article{marino2024input,
	title={Input State Stability of Gated Graph Neural Networks},
	author={Marino, Antonio and Pacchierotti, Claudio and Giordano, Paolo Robuffo},
	journal={IEEE Transactions on Control of Network Systems},
	pages={1--12},
	year={2024}
}

@inproceedings{bayer2013discrete,
  title={Discrete-time incremental ISS: A framework for robust NMPC},
  author={Bayer, Florian and B{\"u}rger, Mathias and Allg{\"o}wer, Frank},
  booktitle={2013 European Control Conference (ECC)},
  pages={2068--2073},
  year={2013},
  
}

@article{li2021message,
  title={Message-aware graph attention networks for large-scale multi-robot path planning},
  author={Li, Qingbiao and Lin, Weizhe and Liu, Zhe and Prorok, Amanda},
  journal={IEEE Robotics and Automation Letters},
  volume={6},
  number={3},
  pages={5533--5540},
  year={2021},
  
}

@inproceedings{tolstaya2021multi,
  title={Multi-robot coverage and exploration using spatial graph neural networks},
  author={Tolstaya, Ekaterina and Paulos, James and Kumar, Vijay and Ribeiro, Alejandro},
  booktitle={2021 IEEE/RSJ International Conference on Intelligent Robots and Systems (IROS)},
  pages={8944--8950},
  year={2021}
}

@inproceedings{gama2021graph,
  title={Graph neural networks for decentralized controllers},
  author={Gama, Fernando and Tolstaya, Ekaterina and Ribeiro, Alejandro},
  booktitle={ICASSP 2021-2021 IEEE International Conference on Acoustics, Speech and Signal Processing (ICASSP)},
  pages={5260--5264},
  year={2021},
  
}

@inproceedings{I-Malli2021CE,
  title={Robust Distributed Estimation of the Algebraic Connectivity for Networked Multi-robot Systems},
  author={Malli, Ioanna and Bechlioulis, Charalampos P and Kyriakopoulos, Kostas J},
  booktitle={2021 IEEE International Conference on Robotics and Automation (ICRA)},
  pages={9155--9160},
  year={2021},
  
}

@incollection{I-Lingfei2022GNNFoundations,
  title={Graph neural networks},
  author={Wu, Lingfei and Cui, Peng and Pei, Jian and Zhao, Liang and Song, Le},
  booktitle={Graph Neural Networks: Foundations, Frontiers, and Applications},
  pages={27--37},
  year={2022},
  publisher={Springer}
}

@article{P-Shuman2013SPG,
  title={The emerging field of signal processing on graphs: Extending high-dimensional data analysis to networks and other irregular domains},
  author={Shuman, David I and Narang, Sunil K and Frossard, Pascal and Ortega, Antonio and Vandergheynst, Pierre},
  journal={IEEE Signal Processing Magazine},
  volume={30},
  number={3},
  pages={83--98},
  year={2013},
  
}

@article{robuffo2013passivity,
  title={A passivity-based decentralized strategy for generalized connectivity maintenance},
  author={{Robuffo Giordano}, Paolo and Franchi, Antonio and Secchi, Cristian and B{\"u}lthoff, Heinrich H},
  journal={The International Journal of Robotics Research},
  volume={32},
  number={3},
  pages={299--323},
  year={2013},
  publisher={SAGE Publications Sage UK: London, England}
}

@article{C-Freeman2010ConnectivityEstimation,
  title={Decentralized estimation and control of graph connectivity for mobile sensor networks},
  author={Yang, Peng and Freeman, Randy A and Gordon, Geoffrey J and Lynch, Kevin M and Srinivasa, Siddhartha S and Sukthankar, Rahul},
  journal={Automatica},
  volume={46},
  number={2},
  pages={390--396},
  year={2010}
}

\appendix

\section{Proof Theorem~\ref{dISS_ed_stab}}
\label{proof_theorem1}
\begin{proof}
	To prove the assertion is sufficient to compute the superior extreme of the difference between two states $\bm{x}_1$ and $\bm{x}_2$. First, we name $ A_{edS},B_{edS},\hat{A}_{edS},\tilde{A}_{edS},\hat{B}_{edS},\tilde{B}_{edS}$ the graph filters with encoder-decoder communication. In light of assumption~\ref{assumption1}, the state gate can be bounded by
	\begin{equation*}
		\begin{aligned}
			|| \bm{\hat{q}} ||_{\infty} & \leq \max_{u \in \mathcal{U}, x \in \mathcal{X} } || \sigma(\hat{A}_{edS}(\bm{x}) + \hat{B}_{edS}(\bm{u}) + \hat{b})||_\infty \\
			& \leq \sigma((KD_x||S||_{\infty}(E_x + Eb_x) + Db_x)||\hat{A}||_{\infty} + (KD_u||S||_{\infty}(E_u + Eb_u) + Db_u)||\hat{B}||_{\infty}+\hat{b})=\sigma_{\hat{q}ed}
		\end{aligned}
	\end{equation*}
	In the same way, the input gate is bounded by
	\begin{equation*}
		\begin{aligned}
			|| \bm{\tilde{q}}||_{\infty} & \leq \sigma((KD_x||S||_{\infty}(E_x + Eb_x) + Db_x)||\tilde{A}||_{\infty} + (KD_u||S||_{\infty}(E_u + Eb_u) + Db_u)||\tilde{B}||_{\infty}+\tilde{b}) = \sigma_{\tilde{q}ed}
		\end{aligned}
	\end{equation*}
	By appling the subadditivity property of the infinite norm, the update of the states $\bm{x}_1$ and $\bm{x}_2$ lead to the following norm inequality 
	\begin{equation}
		\begin{aligned}
			&||\bm{x_1}^+ - \bm{x_2}^+||_{\infty} \leq  \\ & ||\hat{q}_1 \circ (A_{edS1}(\bm{x_1}) - A_{edS2}(\bm{x_2}))||_{\infty} + ||(\hat{q}_1 - \hat{q}_2) \circ A_{edS2}(\bm{x_2})||_{\infty} +\\ &  ||\tilde{q}_1 \circ (B_{edS1}(\bm{u_1}) - B_{edS2}(\bm{u_2}))||_{\infty} + ||(\tilde{q}_1 - \tilde{q}_2) \circ B_{edS2}(\bm{u_2})||_{\infty}
		\end{aligned}
		\label{eq:state_diff}
	\end{equation}
	We want to emphasize that, based on the assumptions regarding the non-linearities in encoding and decoding, the encoding/decoding functions exhibit Lipschitz continuity with a Lipschitz constant of $1$. Then, under the assumptions of theorem~\ref{dISS_ed_stab}, i.e. ${ ||S_{K1}||_{\infty},||S_{K2}||_{\infty} \leq ||\bar{S}_{K}||_{\infty}}$, the first term of the right-hand side satisfies the following
	\begin{equation}
		\begin{aligned}
			||\hat{q}_1 \circ & (A_{edS1}(\bm{x_1}) - A_{edS2}(\bm{x_2}))||_{\infty} \leq  \sigma_{\hat{q}ed}||A||_{\infty}||[d_{x\theta}(e_{x\theta}(\bm{x_1})), \dots, d_{x\theta}(S^K e_{x\theta}(\bm{x_1}))] \\ & -  [d_{x\theta}(e_{x\theta}(\bm{x_2})) \dots, d_{x\theta}(S^K e_{x\theta}(\bm{x_2}))]||_{\infty} \\ & \leq \sigma_{\hat{q}ed} ||A||_{\infty} K D_x || S_{K1} ( I_{K} \otimes e_{x\theta}(\bm{x_1}) - I_{K} \otimes e_{x\theta}(\bm{x_2)}) + (S_{K1} - S_{K2})(I_{K} \otimes e_{x\theta}(\bm{x_2)})||_{\infty} \\ & \leq  
			\sigma_{\hat{q}ed} ||A||_{\infty} K D_x (|| \bar{S}_{K}||_{\infty} E_x || \bm{x_1} - \bm{x_2}||_{\infty} + ||S_{K1} - S_{K2}||_{\infty}(E_x||\bm{x}_2||_{\infty}+ Eb_x))  \\ & \leq \sigma_{\hat{q}ed} ||A||_{\infty} K D_x (|| \bar{S}_{K}||_{\infty} E_x || \bm{x_1} - \bm{x_2}||_{\infty} + ||S_{K1} - S_{K2}||_{\infty}(E_x + Eb_x))
		\end{aligned}
		\label{eq:state_ineq}
	\end{equation}
	Applying the same reasoning to the other terms in the inequality~\eqref{eq:state_diff}, we get
	\begin{equation*}
		\begin{aligned}
			||\bm{x_1}^+ - \bm{x_2}^+||_{\infty} & \leq K D_x E_x ||\bar{S}_K||_\infty(\sigma_{\hat{q}ed} ||A||_{\infty}  +\frac{1}{4}(||\hat{A}||_{\infty} ||A||_{\infty}(K D_x||\bar{S}_K||_\infty (E_x + Eb_x) + Db_x) \\
			& + ||\tilde{A}||_{\infty} ||B||_{\infty}(K D_u||\bar{S}_K||_\infty (E_u + Eb_u) + Db_u))) ||\bm{x_1} - \bm{x_2}||_{\infty} + \mathcal{W}_{ed}||S_{K1}-S_{K2}||_{\infty} \\ & +  K D_u E_u ||\bar{S}_K||_\infty ( \sigma_{\tilde{q}ed}||B||_{\infty} +\frac{1}{4}(||\hat{B}||_{\infty}||A||_{\infty}(K D_x||\bar{S}_K||_\infty (E_x + Eb_x) + Db_x) \\ & + ||\tilde{B}||_{\infty}||B||_{\infty}(K D_u||\bar{S}_K||_\infty (E_u + Eb_u)+  Db_u)))||\bm{u_1} - \bm{u_2}||_{\infty} \\ & \leq 
			\delta \mathcal{A}_{ed} ||\bm{x_1} - \bm{x_2}||_{\infty} + \delta \mathcal{B}_{ed} ||\bm{u_1} - \bm{u_2}||_{\infty} +  \mathcal{W}_{ed}||S_{K1}-S_{K2}||_{\infty}
		\end{aligned} 
	\end{equation*}
	where $ \mathcal{W}_{ed}$ gathers all the coefficient multiplying the difference $ ||S_{K1}-S_{K2}||_{\infty}$. We can consider this latter as an additional bounded input, $ ||S_{K1} - S_{K2}||_{\infty} \leq ||\bar{S}_{K}||_{\infty} - 1$ which is defined by the team state. Hence, by iterating~\eqref{eq:state_ineq} for $ t$ steps, it holds 
	\begin{equation}
		\begin{aligned}
			||\bm{x_1}(t) &- \bm{x_2}(t)||_{\infty} \leq \delta \mathcal{A}^t_{ed} ||\bm{x_1}(0) - \bm{x_2}(0)||_{\infty} + (1-\delta \mathcal{A}_{ed})^{-1} \delta \mathcal{B}_{ed} ||\bm{u_1} - \bm{u_2}||_{\infty} + \\ & (1-\delta \mathcal{A}_{ed})^{-1} \mathcal{W}_{ed}||S_{K1}-S_{K2}||_{\infty}
		\end{aligned}
		\label{eq:incremental_trajectories}
	\end{equation}
	which satisfies the incrementally ISS under $  \delta \mathcal{A}_{ed} < 1$.
\end{proof}

\section{Additional Results}
We reported in Figures~\ref{fig:connectivity_control-pp} and \ref{fig:formation-pp} the results for connectivity maintenance and formaction control problems using \textit{PI-ACE} and \textit{PP-ACE} algorithms. The results of \textit{PI-ACE} are clearly very similar to \textit{R-ACE}, given the similarities of these two estimators, with worst performaces of \textit{PI-ACE} in terms of convergence speed and error. On the contrary, \textit{PP-ACE} seams not suited not suitable for practicale applications as shown by its worst performaces. This result is in contrast on what is shown by the authors of the algorithm who used the \textit{PP-ACE} to estimated the connectivity and control the team to mantain the connectivity of the team~\cite{I-Malli2021CE}. However, the authors in their experiments use high gains and a stepest prescribed function thanks to the numerical integrations used and instantaneous communication. In practice, we add to limit the gains and the function to make this algorithm feasable in a more realistic scenario. 
    
\begin{figure*}[t]
	\begin{subfigure}[b]{0.45\textwidth}
		\centering
		\includegraphics[width=\textwidth]{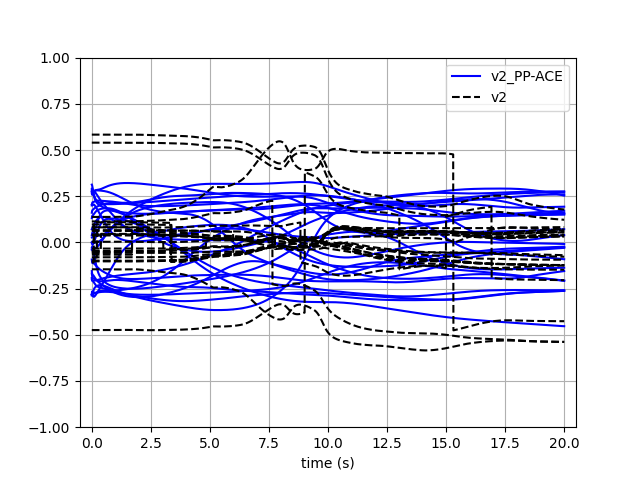}
		\caption{$v_2$ tracking vector with \textit{PP-ACE} model}
	\end{subfigure}
	\begin{subfigure}[b]{0.45\textwidth}
		\centering
		\includegraphics[width=\textwidth]{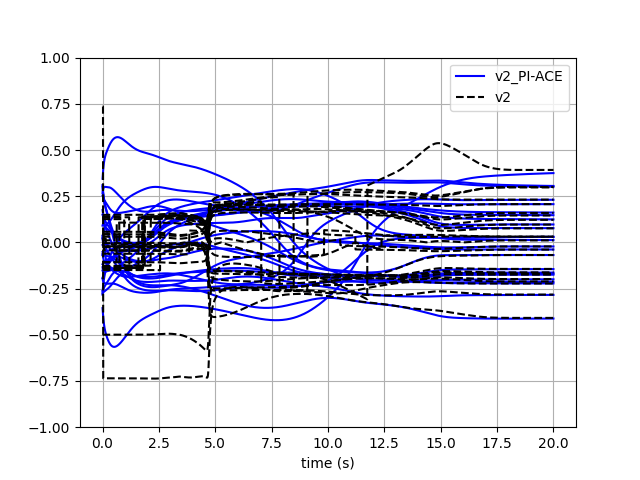}
		\caption{$v_2$ tracking vector with \textit{PI-ACE} model}
	\end{subfigure}
	\caption{\textbf{Connectivity Maintenance}. Laplacian second eigenvector $v2$ estimated during the connectivity control thanks to the average estimation, using (a) \textit{PP-ACE} and (b) \textit{PI-ACE}.}
	\label{fig:connectivity_control-pp}
\end{figure*}

\begin{figure*}[t]
	\begin{subfigure}[b]{0.45\textwidth}
		\centering
		\includegraphics[width=\textwidth]{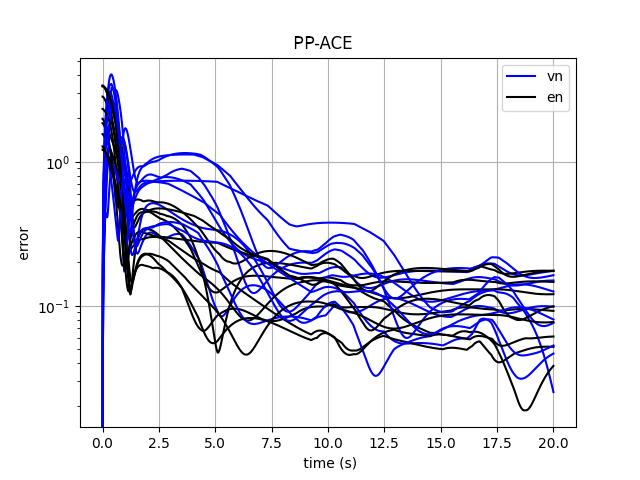}
		\caption{tracking error for \textit{PP-ACE}}
	\end{subfigure}
	\begin{subfigure}[b]{0.45\textwidth}
		\centering
		\includegraphics[width=\textwidth]{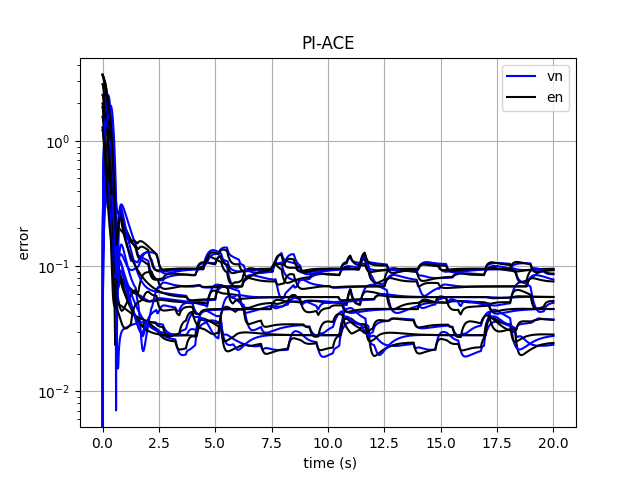}
		\caption{tracking error for \textit{PI-ACE}}
	\end{subfigure}
	\hfill
	\caption{\textbf{Formation Control}. Error in logarithmic scale for the estimation of the average velocity $vn$ and the average agent-target distance $en$. The figures show the errors with 6-dimensional and 1-dimensional input features by using (a) the \textit{PP-ACE} and (b) \textit{PI-ACE}.}
	\label{fig:formation-pp}
\end{figure*}

\end{document}